\documentclass{article}

\usepackage[preprint]{neurips_2026}

\usepackage[utf8]{inputenc} % allow utf-8 input
\usepackage[T1]{fontenc}    % use 8-bit T1 fonts
\usepackage{hyperref}       % hyperlinks
\usepackage{url}            % simple URL typesetting
\usepackage{booktabs}       % professional-quality tables
\usepackage{amsfonts}       % blackboard math symbols
\usepackage{nicefrac}       % compact symbols for 1/2, etc.
\usepackage{microtype}      % microtypography
\usepackage{xcolor}         % colors
\usepackage{amsthm}         % theorem environments

\usepackage{wrapfig}
\usepackage{inconsolata}
\usepackage{tcolorbox}

\usepackage{graphicx}
\usepackage{multirow}
\usepackage{xcolor} 
\usepackage{listings}
\usepackage{multicol} % for multiple columns
\usepackage{lipsum} % for dummy text
\usepackage{amsmath}
\usepackage{amssymb}
\usepackage{algorithm}
\usepackage{algorithmicx}
\usepackage{algpseudocode}
\usepackage{array}
\usepackage{booktabs}
\usepackage{arydshln}
\usepackage{caption}
\usepackage{mathtools}

\usepackage{xcolor}
\usepackage{bm}       
\usepackage{enumitem}
\usepackage{bbm} 
\usepackage{listings}

\newtheorem{theorem}{Theorem}
\newtheorem{corollary}[theorem]{Corollary}

\theoremstyle{definition}
\newtheorem{definition}[theorem]{Definition}
\newtheorem{assumption}[theorem]{Assumption}

\theoremstyle{remark}

\title{On Time, Within Budget: Constraint-Driven Online Resource Allocation for Agentic Workflows}

\begin{document}

\author {
    \textbf{Xinglin Wang}\textsuperscript{\rm 1}\footnotemark[1], \hspace{0cm}
    \textbf{Zishen Liu}\textsuperscript{\rm 1}\footnotemark[1], \hspace{0cm} 
    \textbf{Shaoxiong Feng}\textsuperscript{\rm 2}\footnotemark[2], \hspace{0cm}
    \textbf{Peiwen Yuan}\textsuperscript{\rm 1}, \hspace{0cm}
    \textbf{Yiwei Li}\textsuperscript{\rm 1}, \hspace{0cm}
    \\
    \textbf{Jiayi Shi}\textsuperscript{\rm 1}, \hspace{0cm}
    \textbf{Yueqi Zhang}\textsuperscript{\rm 1}, \hspace{0cm}
    \textbf{Chuyi Tan}\textsuperscript{\rm 1}, \hspace{0cm}
    \textbf{Ji Zhang}\textsuperscript{\rm 1}, \hspace{0cm}
    \textbf{Boyuan Pan}\textsuperscript{\rm 2}, \hspace{0cm} 
    \textbf{Yao Hu}\textsuperscript{\rm 2}, \hspace{0cm} 
    \textbf{Kan Li}\textsuperscript{\rm 1}\footnotemark[2] \\
    \textsuperscript{\rm 1} School of Computer Science, Beijing Institute of Technology \\
    \textsuperscript{\rm 2} Xiaohongshu Inc \\
    \texttt{\{wangxinglin,liuzishen,peiwenyuan,liyiwei\}@bit.edu.cn} \\
    \texttt{\{zhangyq,shijiayi,tanchuyi,likan\}@bit.edu.cn} \\
    \texttt{shaoxiongfeng2023@gmail.com} \quad \texttt{\{panboyuan,xiahou\}@xiaohongshu.com}
}

\renewcommand{\thefootnote}{\fnsymbol{footnote}}
\footnotetext[1]{Equal contribution.}
\footnotetext[2]{Corresponding author.}
\renewcommand{\thefootnote}{\arabic{footnote}}

\maketitle

\begin{abstract}

Agentic systems increasingly solve complex user requests by executing orchestrated workflows, where subtasks are assigned to specialized models or tools and coordinated according to their dependencies. 
While recent work improves agent efficiency by optimizing the performance--cost--latency frontier, real deployments often impose concrete requirements: a workflow must be completed within a specified budget and before a specified deadline.
This shifts the goal from average efficiency optimization to maximizing the probability that the entire workflow completes successfully under explicit budget and deadline constraints. 
We study \emph{constraint-driven online resource allocation for agentic workflows}. 
Given a dependency-structured workflow and estimates of success rates and generation lengths for each subtask--model pair, the executor dynamically allocates models and parallel samples across simultaneously executable subtasks while managing the remaining budget and time.
We formulate this setting as a finite-horizon stochastic online allocation problem and propose \emph{Monte Carlo Portfolio Planning} (MCPP), a lightweight closed-loop planner that directly estimates constrained completion probability through simulated workflow executions and replans after observed outcomes. 
Experiments on CodeFlow and ProofFlow demonstrate that MCPP consistently improves constrained completion probability over strong baselines across a wide range of budget--deadline constraints\footnote{Our code and data have been released on \url{https://github.com/WangXinglin/MCPP}.}.

\end{abstract}

\section{Introduction}

Recent advances in large language models (LLMs) have substantially expanded the capabilities of agentic systems, enabling planning and acting in interactive environments \citep{yao2022react,liuagentbench,erdoganplan}, tool use \citep{schick2023toolformer,qintoolllm,patil2024gorilla,ahnorchestrationbench}, and complex multi-step task execution \citep{mialongaia,jimenezswe, zhang2026clawbench}.
To handle open-ended and long-horizon objectives, recent agentic systems increasingly rely on workflow orchestration to expose parallelism, exploit specialization, and manage long-range dependencies: a user request is decomposed into subtasks, assigned to specialized agents, models, or tools, and executed concurrently when dependencies allow \citep{zhuge2024gptswarm,zhang2025agentorchestra,team2026kimi}.
This shift turns agent execution from an isolated model response into a graph-structured workflow, making the organization of subtasks and their information flow an optimizable object in agentic systems \citep{zhang2025multi,zhangg,zhangaflow,wang2025evoagentx,wang2025scoreflow}.

Yet a well-designed workflow does not by itself determine how the executor should carry it out once execution begins, including which model to invoke for each currently executable subtask and how many parallel samples to draw \citep{laju2026nalar}.
Existing work typically studies these execution-time choices from an efficiency-oriented perspective, adapting model selection, routing decisions, or inference behavior to improve the performance--cost--latency frontier \citep{wangagenttts,zhang2026evoroute,ma2026timely,fan2026timebill}.
However, users in real deployments rarely ask whether an agent is generally faster, cheaper, or more accurate on average.
They ask whether this particular workflow can be completed within a specified budget \(B\) and before a specified deadline \(D\) (e.g., polishing a NeurIPS submission within eight hours under a \$100 budget).
A better performance--cost--latency frontier still does not answer this instance-specific question, as it does not specify how to spend the remaining budget and time at each workflow state to maximize the chance of finishing within \(B\) and \(D\).
Under such requirements, cost and time become part of the success condition, shifting the objective from improving performance--cost--latency frontiers to maximizing the probability that the entire workflow completes successfully within the specified budget and deadline (Figure~\ref{fig:intro_overview}).

% Yet even a well-designed workflow does not by itself determine how the system should spend its execution resources once the workflow begins \citep{laju2026nalar}.
% At each execution state, the executor must still decide which model to use for each ready subtask, how many parallel samples to draw, and how much budget to spend now versus preserve for downstream dependencies.
% Existing work addresses aspects of this execution-time allocation problem from an efficiency-oriented perspective, adapting compute allocation, model choices, routing decisions, or inference behavior to improve task success while reducing monetary cost and wall-clock delay \citep{wangagenttts,zhang2026evoroute,ma2026timely,fan2026timebill}.
% However, users in real deployments rarely ask whether an agent is generally faster, cheaper, or more accurate on average.

\begin{figure*}[t]
    \centering
    \includegraphics[width=\textwidth]{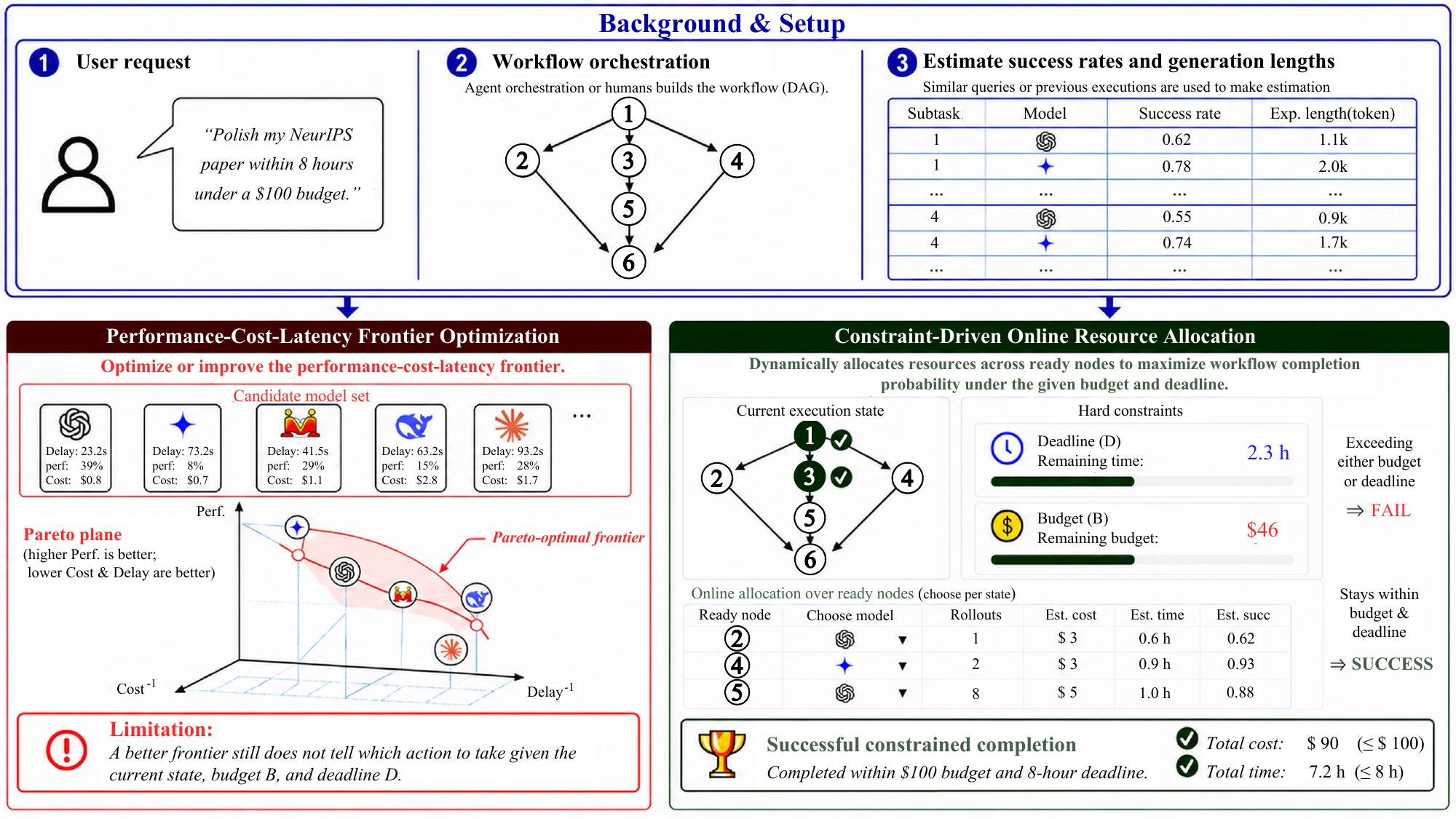}
    \caption{
    Comparison between performance--cost--latency frontier optimization and constraint-driven online resource allocation.
    A user request is first transformed into a static workflow, and success-rate and generation-length estimates for subtask--model pairs can be obtained from similar past queries or executions.
    Previous methods optimize or improve the performance--cost--latency frontier, but a better frontier still does not specify which action to take given the current workflow state, budget \(B\), and deadline \(D\).
    In contrast, our method dynamically allocates models and parallel rollouts to ready subtasks during execution to maximize the probability that the workflow completes successfully within the specified budget and deadline.
    }
    \label{fig:intro_overview}
\end{figure*}

Motivated by this gap, we study \emph{constraint-driven online resource allocation for agentic workflows}. 
The workflow itself may be generated by an orchestrator, designed by humans, or optimized by an external workflow search method. 
Our focus begins once this workflow is given and execution starts. 
We assume that, for each subtask--model pair, the executor has estimates of single-attempt success probability and expected generation length\footnote{Estimating these quantities is complementary to our goal.}, which can be obtained from historical executions, offline calibration, or learned prediction methods \citep{pacchiardi2024100,pan2025route,ding2025best,zhang2026evoroute}.

Given these estimates, the remaining challenge is to decide how to allocate resources during execution. 
This is a non-trivial online problem: each allocation changes the probability of completing currently ready subtasks, consumes budget and wall-clock time, and affects which downstream subtasks may become available later. 
A static assignment of models or sampling widths is therefore insufficient, since the best decision depends on the evolving workflow state, the remaining budget, the remaining time, and the stochastic success or failure events observed so far. 
We formalize the objective as maximizing the \emph{constrained completion probability},
\[
    \Pr_{\pi}\!\left(T_{\mathrm{solve}}\le D,\; C_{\pi}\le B\right),
\]
where \(T_{\mathrm{solve}}\) is the workflow completion time and \(C_{\pi}\) is the total execution cost under policy \(\pi\). 
Although exact dynamic programming gives a clean conceptual solution, it quickly becomes intractable as the completed-subtask set, allocation choices over currently executable subtasks, and stochastic execution outcomes grow combinatorially. 
To obtain a practical solver, we propose Monte Carlo Portfolio Planning (MCPP), a simple yet effective online planning framework.
At each execution state, the planner evaluates candidate resource-allocation actions through simulated workflow executions, selects the action with the highest estimated constrained completion probability, observes the actual outcomes, and replans from the updated state. 
This directly targets the user-facing success event instead of optimizing a soft performance--cost--latency proxy.

To validate the effectiveness of our approach, we conduct experiments on CodeFlow~\citep{wang2025codeflowbench} and ProofFlow~\citep{cabral2025proofflow}, two representative benchmarks for dependency-structured code and proof tasks.
Across varying budget--deadline constraints, our method improves constrained completion probability over strong baselines.
We further conduct noise-injection experiments to evaluate the robustness of our method under noisy estimates of subtask success rates and generation lengths.

Our contributions are summarized as follows:
\begin{itemize}
    \item We introduce \emph{constraint-driven online resource allocation for agentic workflows} to capture a common deployment requirement: whether a given workflow can be completed successfully within a specified budget and deadline. In this setting, budget and deadline are treated as hard success conditions, and the objective is to maximize constrained completion probability.

    \item We formulate this setting as a finite-horizon stochastic online allocation problem and propose \emph{Monte Carlo Portfolio Planning} (MCPP), a closed-loop planner with a safe-improvement guarantee that directly estimates constrained completion probability to allocate models and parallel samples online.

    \item We empirically evaluate MCPP on CodeFlow and ProofFlow, demonstrating consistent gains over strong baselines across diverse budget--deadline constraints, and further validate its robustness under noisy estimates of subtask success rates and generation lengths.
\end{itemize}

\section{Related Work}

\textbf{Resource-Aware Agent Execution.}
As LLM systems move from single-turn responses to long-horizon agentic execution, a central question is how to obtain strong task performance without invoking expensive computation at every step. 
Early work studies this problem at the query level: routing and cascading methods select among models, or call them sequentially, to balance response quality and inference cost~\citep{chenfrugalgpt,dinghybrid,ongroutellm,wang2025mixllm,shen2025sater, jali2026not}. 
This query-level view has recently been extended to more structured and agentic settings, where resource decisions must be made repeatedly during execution. 
AgentTTS allocates test-time compute across stages of complex tasks~\citep{wangagenttts}, and EvoRoute dynamically selects LLM backbones for agentic subtasks to improve the performance--cost--latency frontier~\citep{zhang2026evoroute}. 
Another thread makes latency and wall-clock time explicit, showing that token-efficient inference and time-efficient execution can lead to different strategies~\citep{huang2025latency,wang2025faster,ma2026timely,fan2026timebill}. 
More recent agentic methods further introduce explicit budget control, including budget-aware tool scaling, strict-budget tool-use planning, and budget-aware model routing~\citep{liu2025budget,liu2026budget,zhang2026budget}.

This line of work makes agent execution increasingly resource-aware, but it typically optimizes decisions along a single sequential run, such as choosing the next model, tool, compute budget, or inference behavior based on the current history.
Our work studies a different decision unit: Given a dependency-structured workflow, multiple subtasks may become executable simultaneously, and the executor must allocate base models and parallel samples across them while jointly managing the remaining budget and time.
We therefore optimize the probability that the entire workflow completes within the specified constraints, rather than a per-step routing decision or an average efficiency frontier.

\textbf{Performance and Length Estimation.}
Our framework assumes access to estimates of single-attempt success probability and expected generation length for each subtask--model pair.
These quantities are closely related to two lines of work: performance estimation for deciding which model or strategy to use, and length estimation for predicting the cost and latency of an execution.
LLM routing methods commonly estimate model quality, query difficulty, confidence, or strategy utility to decide which model or reasoning procedure should be selected~\citep{dinghybrid,wang2025mixllm,pan2025route,ding2025best,shen2025sater}.
Some methods further adapt the amount of test-time computation, such as choosing the number of sampled responses or selecting among decoding strategies based on predicted performance, token usage, cost, or latency~\citep{ding2025best,huang2025latency}.
Complementary work predicts model success from historical evaluations, small reference sets, or online feedback, reducing the cost of estimating how likely a model is to solve unseen instances~\citep{pacchiardi2024100,panda2025adaptive}.
These studies provide natural sources for the success-rate and generation-length estimates used in our setting.
Rather than improving these estimators themselves, we study how to use their outputs for online resource allocation over a dependency-structured workflow, and we evaluate robustness when the estimates are noisy.

\textbf{Workflow Orchestration.}
A separate line of work studies the upstream question of how agentic workflows should be designed, generated, or optimized. 
Graph-based formulations represent agents, tools, and model calls as nodes connected by information-flow or dependency edges, turning agent systems into optimizable computational graphs~\citep{zhuge2024gptswarm}. 
Automated workflow search methods such as AFlow optimize executable workflow programs using task feedback~\citep{zhangaflow}, while subsequent work explores multi-agent architecture search, communication topology design, dynamic agent elimination, workflow evolution, and preference-based workflow optimization~\citep{zhang2025multi,zhangg,wang2025agentdropout,wang2025evoagentx,wang2025scoreflow}. 
System-oriented efforts further emphasize hierarchical orchestration, workflow serving, distributed execution, and runtime control for long-horizon agent applications~\citep{zhang2025agentorchestra,team2026kimi,laju2026nalar,zhang2026megaflow}. 
Together, this literature establishes workflow structure as a first-class object in agentic systems. 
Our work is orthogonal to workflow construction: we begin after a workflow has been specified, and study how to execute it under explicit budget and deadline constraints.

\section{Constraint-Driven Online Resource Allocation}
\label{sec:formulation}

We now formalize the execution-time allocation problem studied in this work.
The workflow is assumed to be fixed before execution begins, which be produced by a human designer, an agentic orchestrator, or an automated workflow search method.
Our focus is the following online decision problem:

\begin{quote}
    \emph{Given a workflow, a budget \(B\), a deadline \(D\), and success-rate and generation-length estimates for each subtask--model pair, how should the executor allocate models and parallel samples to currently executable subtasks so as to maximize the probability that the entire workflow completes within \(B\) and \(D\)?}
\end{quote}

\subsection{Problem Formulation}

\paragraph{Workflow, models, and estimates.}
A workflow execution instance is defined as
\[
    \mathcal I=(G,\mathcal M,\Phi,B,D),
\]
where \(G=(V,E)\) is a directed acyclic graph of subtasks, \(\mathcal M\) is the set of available base models, \(\Phi\) denotes estimated success and length statistics, \(B\) is the total budget, and \(D\) is the deadline.
Each node \(v\in V\) is a subtask, and an edge \((u,v)\in E\) means that \(v\) can only be attempted after \(u\) has been completed.
For a completed set \(S\subseteq V\), the currently executable subtasks are
\[
    R(S)=\{v\in V\setminus S:\mathrm{Pred}(v)\subseteq S\},
\]
where \(\mathrm{Pred}(v)\) denotes the predecessors of \(v\).
The workflow is complete when \(S=V\).

For each subtask--model pair \((v,m)\), we assume estimated or calibrated statistics
\[
    \phi_{v,m}=(p_{v,m},\ell_{v,m}),
\]
where \(p_{v,m}\in[0,1]\) is the single-attempt completion probability and \(\ell_{v,m}\) is the expected generation length.
We write these estimated quantities without hats for notational simplicity.
The generation length is converted into per-attempt cost and latency using model-specific pricing and throughput, denoted by \(c_{v,m}\) and \(\tau_{v,m}\), respectively.
If \(k\) independent samples are launched in parallel for subtask \(v\) using model \(m\), the probability that at least one sample succeeds is
\[
    q_{v,m}(k)=1-(1-p_{v,m})^k.
\]
Because the marginal gain of additional samples decreases geometrically, we restrict sampling widths to a finite log-scale set \(\mathcal K\) of positive values (Section~\ref{sec:Onlinesim} gives concrete choice of \(\mathcal K\)).

\paragraph{Online state, action, and transition.}
At any execution step, the state is
\[
    s=(S,b,h),
\]
where \(S\subseteq V\) is the set of completed subtasks, \(b\) is the remaining budget, and \(h\) is the remaining time before the deadline.
The initial state is \(s_0=(\emptyset,B,D)\).
At state \(s=(S,b,h)\), an allocation action assigns a model and a sampling width to every currently executable subtask:
\[
    a=\{(m_v,k_v)\}_{v\in R(S)},
    \qquad
    m_v\in\mathcal M,\quad k_v\in\mathcal K.
\]
Thus, the log-scale action space is
\[
    \mathcal A_{\mathcal K}(S)
    =
    \left\{
        \{(m_v,k_v)\}_{v\in R(S)}
        :
        m_v\in\mathcal M,\;
        k_v\in\mathcal K
    \right\}.
\]
The budget consumed by action \(a\) is
\[
    C(a)=\sum_{v\in R(S)} k_v c_{v,m_v}.
\]
Since currently executable subtasks are executed concurrently, the wall-clock duration is
\[
    \Delta(a)=\max_{v\in R(S)} \Delta_{v,m_v}(k_v),
\]
where \(\Delta_{v,m_v}(k_v)\) is the estimated duration of launching \(k_v\) parallel samples for subtask \(v\) with model \(m_v\).
Under ideal parallelism, \(\Delta_{v,m_v}(k_v)\) is close to the single-attempt latency \(\tau_{v,m_v}\), while more general systems may use an estimated batch-latency model.

If \(C(a)>b\) or \(\Delta(a)>h\), the action violates the remaining constraints and has constrained completion value zero.
Otherwise, after executing \(a\), each currently executable subtask \(v\in R(S)\) completes independently with probability \(q_{v,m_v}(k_v)\).
Let \(U\subseteq R(S)\) be the subset of subtasks completed in this round. Then
\[
    \Pr(U\mid s,a)
    =
    \prod_{v\in U} q_{v,m_v}(k_v)
    \prod_{v\in R(S)\setminus U}
    \left(1-q_{v,m_v}(k_v)\right),
\]
and the next state is
\[
    s'=
    \bigl(
        S\cup U,\,
        b-C(a),\,
        h-\Delta(a)
    \bigr).
\]
This transition captures the core difficulty of workflow execution: one allocation can make partial progress, consume budget and time, and unlock different downstream subtasks depending on stochastic outcomes.

\paragraph{Objective and exact Bellman backup.}
A policy \(\pi\) maps each state \(s=(S,b,h)\) to an allocation action.
Let \(T_{\mathrm{solve}}\) be the wall-clock time at which all subtasks are completed, and let \(C_\pi\) be the total execution cost under policy \(\pi\).
We aim to maximize the constrained completion probability
\[
    J(\pi)
    =
    \Pr_{\pi}\!\left(
        T_{\mathrm{solve}}\le D,\;
        C_{\pi}\le B
    \right).
\]
Budget and time are therefore part of the success condition: an execution that exceeds either limit is not merely inefficient, but fails to satisfy the user's request.

Within the log-scale action space \(\mathcal A_{\mathcal K}\), the optimal constrained completion probability admits a dynamic-programming characterization.
Let \(V^\star_{\mathcal K}(S,b,h)\) denote the optimal value from state \((S,b,h)\).
The terminal conditions are
\[
    V^\star_{\mathcal K}(S,b,h)=1
    \quad\text{if } S=V,
\]
and
\[
    V^\star_{\mathcal K}(S,b,h)=0
    \quad\text{if } S\neq V
    \text{ and there is no } a\in\mathcal A_{\mathcal K}(S)
    \text{ such that } C(a)\le b,\ \Delta(a)\le h.
\]
For all other states,
\[
    V^\star_{\mathcal K}(S,b,h)
    =
    \max_{a\in\mathcal A_{\mathcal K}(S)}
    Q^\star_{\mathcal K}(s,a),
\]
where
\[
    Q^\star_{\mathcal K}(s,a)
    =
    \begin{cases}
    \displaystyle
    \sum_{U\subseteq R(S)}
    \Pr(U\mid s,a)
    V^\star_{\mathcal K}
    \bigl(
        S\cup U,\,
        b-C(a),\,
        h-\Delta(a)
    \bigr),
    & C(a)\le b,\ \Delta(a)\le h,\\[2ex]
    0,
    & \text{otherwise}.
    \end{cases}
\]

This Bellman backup makes the allocation trade-off explicit.
Using stronger models or more samples can increase immediate progress, but it also consumes budget and time that may be needed by downstream subtasks.
Exact dynamic programming is conceptually clean but impractical: the completed set \(S\) can take exponentially many values in \(|V|\), the action space scales as
\[
    (|\mathcal M|\,|\mathcal K|)^{|R(S)|},
\]
and each action induces a distribution over \(2^{|R(S)|}\) possible success subsets.
This combinatorial barrier motivates an online approximation that preserves the constrained completion objective without exhaustive Bellman evaluation.

\subsection{Monte Carlo Portfolio Planning}

To obtain a practical solver, we propose \emph{Monte Carlo Portfolio Planning} (MCPP), a lightweight online planner that approximates the Bellman decision rule with a one-step portfolio rollout backup (Algorithm~\ref{alg:mcpp}).
At state \(s=(S,b,h)\), MCPP constructs a compact candidate action set
\[
    \widetilde{\mathcal A}(s)\subseteq \mathcal A_{\mathcal K}(S).
\]
This set may enumerate the full log-scale action space when few subtasks are executable, or be pruned to a manageable subset when the branching factor is large.
In all cases, it includes the base actions induced by a portfolio of simple continuation policies: \[ \Pi_0= \{\pi_{m,k}:m\in\mathcal M,\; k\in\mathcal K\}. \] Each \(\pi_{m,k}\) is a model-specific Retry-\(k\) policy that assigns model \(m\) and \(k\) parallel samples to every currently executable subtask at each future state: \[ \pi_{m,k}(S,b,h) = \{(m,k)\}_{v\in R(S)}. \] This portfolio spans conservative to aggressive spending behaviors and repeats this coverage for each available base model.

For a candidate current action \(a\in\widetilde{\mathcal A}(s)\) and continuation policy \(\mu\in\Pi_0\), define
\[
    Q_\mu(s,a)
    =
    \Pr\!\left(
        \text{workflow completes within } b,h
        \mid
        \text{first execute } a,
        \text{ then follow } \mu
    \right).
\]
MCPP estimates this probability by Monte Carlo simulation.
Each simulated rollout first applies \(a\), samples the stochastic completed subset, updates the state, and then follows \(\mu\) until the workflow completes or the constraints are violated.
Let \(\widehat Q_\mu(s,a)\) be the resulting Monte Carlo estimate.
The candidate action is scored by its best continuation policy,
\[
    \widehat Q_{\Pi_0}(s,a)
    =
    \max_{\mu\in\Pi_0}
    \widehat Q_\mu(s,a),
\]
and MCPP selects
\[
    a_{\mathrm{MCPP}}(s)
    =
    \arg\max_{a\in\widetilde{\mathcal A}(s)}
    \widehat Q_{\Pi_0}(s,a).
\]

Only the selected current action is executed in the real workflow.
After observing which subtasks actually completed, the executor updates the state and replans.
Thus, MCPP induces a closed-loop execution policy rather than a fixed model assignment or fixed sampling-width rule.

We provide a theoretical analysis showing that, when the candidate set contains the portfolio-induced base actions, MCPP safely improves over the best base policy, up to finite-sample Monte Carlo error, candidate-set approximation error, and errors in the estimated success rates and generation lengths (Appendix~\ref{app:theory}).

\section{Experiments}
In this section, experimental evaluation results are presented to demonstrate the performance of MCPP for deadline-constrained execution of agentic workflows.
We begin by outlining the experiment setup, followed by the main results comparing the MCPP against baselines, and analysis focusing on its scalability, efficiency, action diversity, runtime overhead, and robustness to noise disturbance.
\subsection{Experimental Setup}
\label{sec:experimental_setup}

We evaluate our method on two dependency-structured workflow benchmarks: ProofFlow~\citep{cabral2025proofflow} and CodeFlow~\citep{wang2025codeflowbench}.
ProofFlow evaluates formal-reasoning workflows in which an input problem is decomposed into dependent formalization and proving subtasks, while CodeFlow evaluates code-generation workflows in which a programming task is decomposed into dependent implementation and verification subtasks.
To accurately evaluate different allocation policies, we first construct empirical execution profiles through large-scale rollouts to capture the statistics of each model--subtask pair, such as success rate and output-token length.
Then, numerous offline experiments are conducted to enable statistically reliable evaluation of stochastic workflow execution to measure completion probability under identical budget--deadline settings.
More details about the offline workflow pool construction are provided in Appendix~\ref{sec:pool build}, while the action space of MCPP for these two benchmarks are detailed in Appendix~\ref{sec:Onlinesim}.

We compare against two representative baselines.
\textbf{Uniform} is a static allocation baseline that distributes the available budget uniformly across subtasks before execution and dispatches the assigned rollouts when subtasks become ready.
It represents a static strategy that respects workflow dependencies and hard budget constraints, but does not adapt to intermediate outcomes.
We report the best Uniform result over available model families.
\textbf{Retry} is an event-driven online baseline that observes execution events and applies a fixed local retry-width rule to ready subtasks.
It represents the natural alternative of online local retrying without downstream planning.
We report the best Retry result over both model families and fixed retry-width choices.

\subsection{Main Results}
\label{sec:main_results}
As shown in Figure~\ref{fig:main_result}, our method has the highest success rate in all budget regimes for both formal-reasoning and code-generation workflows.
The improvement is particularly significant under the tightest budget of $B=\$0.05$.
With deadline-aware lookahead, the proposed method can rapidly reach a high success rate once the deadline becomes moderately feasible, while both baselines remain substantially lower.
Under moderate and loose budgets, i.e., $B=\$1$ and $B=\$20$, our method can also converge rapidly, achieving a success rate close to 1 within the 120-minute deadline.
The gains come from planning under joint budget--deadline constraints.
To complete a workflow under hard constraints, the executor must decide not only whether a ready subtask should be attempted, but also how much budget and time should be spent now versus preserved for downstream dependencies.
Uniform allocation ignores this state-dependent trade-off, while Retry reacts to observed execution events but still uses a fixed local rule.
In contrast, MCPP evaluates current actions by simulating downstream workflow completions under the remaining budget and deadline, allowing the planner to prioritize bottleneck subtasks and preserve resources for future dependencies.
\begin{figure}[t]
    \centering
    \includegraphics[width=0.85\textwidth]{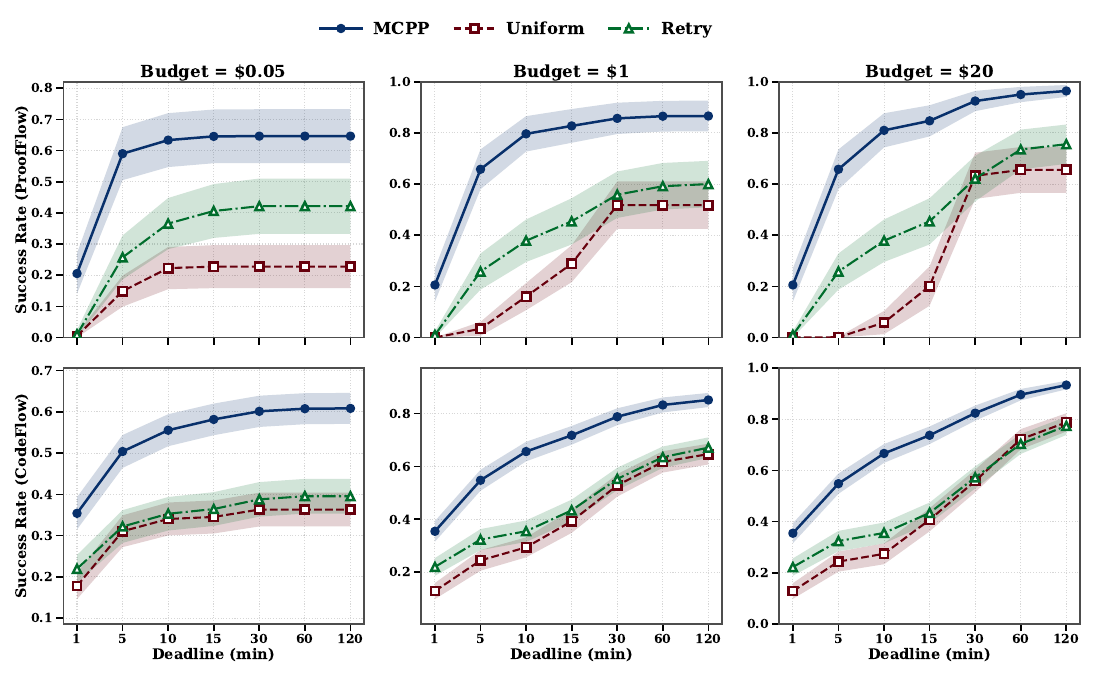}
    \caption{
    \textbf{Main results across workflow benchmarks.}
    We evaluate Monte Carlo Portfolio Planning (MCPP) under the same budget--deadline settings on ProofFlow and CodeFlow.
    }
    \label{fig:main_result}
\end{figure}

\subsection{Analysis}
\label{sec:analysis}

We analyze the proposed method from the following perspectives:
(1) the effect of action-space diversity,
(2) robustness to noisy execution-profile estimates,
(3) scaling behavior with workflow size (Appendix~\ref{sec:node_range}),
(4) runtime overhead (Appendix~\ref{sec:runtime_analysis}),
(5) the effect of Monte Carlo searches (Appendix~\ref{sec:m_sweep}).
Unless otherwise specified, analysis experiments are conducted on ProofFlow.

\paragraph{Effect of Action-Space Diversity}
\label{sec:action_space_diversity}
We first analyze how expanding the planner's action space affects constrained completion probability.
The action space contains two complementary dimensions of model choice and rollout-width choice.
Specifically, model choice determines the success-latency-cost profile of each attempt, while rollout width determines how aggressively the planner samples the selected model on a ready subtask.
Together, these dimensions allow the planner to allocate not only the right model, but also the right amount of computation at each workflow state.

Figure~\ref{fig:model_count} shows that the full three-model portfolio consistently outperforms both the two-model and one-model settings, which indicates that selecting a single globally strongest model is not always sufficient for workflow-level constrained execution.
Since model usefulness is subtask-dependent, expensive but strong models can be valuable for difficult bottleneck nodes, while cheaper models can suffice for easier nodes or tighter budgets.
Figure~\ref{fig:k_sweep} further shows that richer rollout-width portfolios improve performance by allowing the planner to vary compute intensity across nodes and states.
These results show that the gain comes from action diversity: heterogeneous model families provide different success--latency--cost trade-offs, and multiple rollout widths provide different levels of computational investment.
Accordingly, Monte Carlo lookahead then selects among these node--model--rollout actions based on their downstream constrained completion probability.

\begin{figure}[t]
    \centering
    \includegraphics[width=0.9\textwidth]{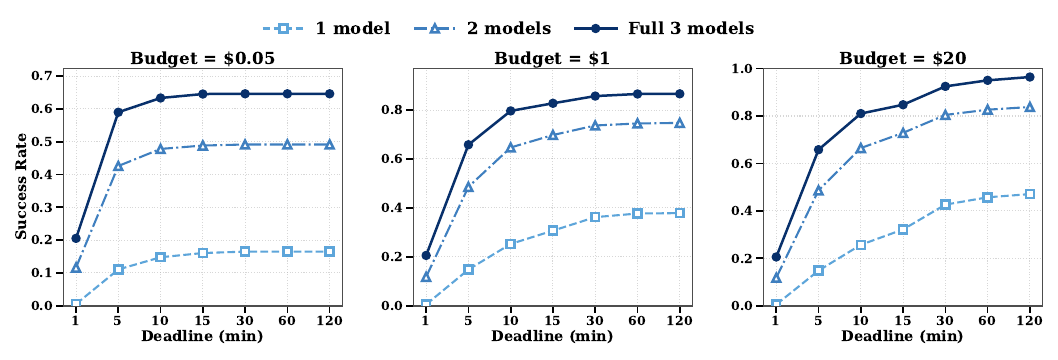}
    \caption{
    \textbf{Effect of model portfolio size.}
        The full three-model portfolio consistently outperforms smaller portfolios, showing that model diversity improves workflow-level constrained execution.
        }
    \label{fig:model_count}
\end{figure}
\begin{figure}[t]
    \centering
    \includegraphics[width=1\textwidth]{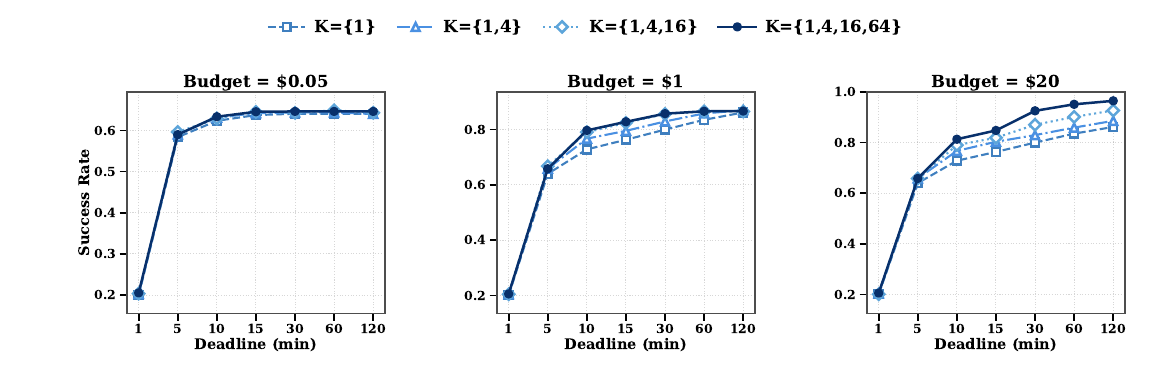}
    \caption{
    \textbf{Effect of rollout-width portfolio.}
    Richer rollout-width portfolios improve performance by allowing the planner to adapt compute intensity across nodes and states.
    }
    \label{fig:k_sweep}
\end{figure}

\paragraph{Robustness to Noisy Estimation}
\label{sec:robustness}

To evaluate robustness to imperfect empirical profiles, we perturb the MCPP planner-visible rollout pools with noise in token length and subtask success rate, while actual execution outcomes are still sampled from the true empirical pools.
For token length, we add Gaussian noise after normalizing each node--model pool by its maximum token length, using noise levels $\sigma\in\{3,4,5\}$.
For success rate, we add Gaussian noise directly to the empirical success rate with $\sigma\in\{0.1,0.2,0.3\}$, and minimally flip planning-pool labels to match the perturbed rate.
The detailed perturbation approaches are provided in Appendix~\ref{app:robustness_details}.

Table~\ref{tab:robustness_noise} shows that MCPP remains robust under substantial planner-side noise. 
Token-length noise causes only mild degradation, while success-rate noise has a larger effect, with the largest drop of 6.01 percentage points under \(\sigma=0.3\).
This is because token-length noise is continuous and can partially average out when costs are accumulated across sampled rollouts, whereas success-probability noise flips discrete success labels and directly changes the completion events used by the planner.
Thus, success-rate noise more strongly affects action-value estimates and action rankings.
\begin{table}[t]
    \centering
    \small
    \renewcommand{\arraystretch}{0.95}
    \setlength{\tabcolsep}{4.5pt}
    \caption{
    \textbf{Robustness to noisy token-length and success-rate profiles.}
    Each cell reports the change in constrained completion probability relative to the clean setting, measured in percentage points.
    }
    \label{tab:robustness_noise}
    \begin{tabular}{lcccccc}
        \toprule
        \textbf{Deadline} &
        \multicolumn{3}{c}{\textbf{Token-length noise}} &
        \multicolumn{3}{c}{\textbf{Success-rate noise}} \\
        \cmidrule(lr){2-4} \cmidrule(lr){5-7}
        & $\sigma=3$ & $\sigma=4$ & $\sigma=5$
        & $\sigma=0.1$ & $\sigma=0.2$ & $\sigma=0.3$ \\
        \midrule
        $10$ min  & $-1.97$ & $-1.81$ & $-1.66$ & $-2.10$ & $-4.01$ & $-6.01$ \\
        $30$ min  & $-1.90$ & $-1.78$ & $-2.09$ & $-2.47$ & $-4.05$ & $-4.61$ \\
        $120$ min & $-1.02$ & $-0.86$ & $-0.84$ & $-0.66$ & $-1.65$ & $-2.76$ \\
        \bottomrule
    \end{tabular}
\end{table}

\section{Conclusions}

In this work, we formulate \emph{constraint-driven online resource allocation for agentic workflows}, a hard-constrained execution setting in which the objective is not to optimize an average performance--cost--latency frontier, but to maximize the probability that a given workflow completes successfully within a specified budget and deadline.
We model this problem as a finite-horizon stochastic online allocation process over dependency-structured workflows, where the executor dynamically allocates base models and parallel samples to currently executable subtasks according to the remaining budget, remaining time, and observed execution outcomes.
To obtain a practical solver, we propose Monte Carlo Portfolio Planning (MCPP), a lightweight closed-loop planner that directly estimates constrained completion probability through simulated workflow executions and replans after each observed outcome.
Theoretically, we provide a safe-improvement analysis showing that MCPP improves over the best base policy in its portfolio up to finite-sample, candidate-set, and estimation errors.
Empirically, experiments on CodeFlow and ProofFlow demonstrate that MCPP consistently improves constrained completion probability over strong baselines across diverse budget--deadline constraints, with further analyses showing robustness to noisy estimates of subtask success rates and generation lengths.

\textbf{Limitations and future directions.}
MCPP relies on estimates of subtask success rates and generation lengths. 
Although our noise-injection experiments suggest robustness to moderate estimation errors, better calibration and uncertainty-aware estimation could further improve deployment reliability.
In addition,  our experiments focus on dependency-structured code and proof workflows. Extending this setting to more open-ended tool-use, web, and multi-agent workflows would further clarify the generality of constraint-driven execution under real deployment constraints.

\clearpage
\bibliography{anthology,custom}
\bibliographystyle{acl_natbib}

\clearpage
\appendix

\section{Monte Carlo Portfolio Planning Algorithm}
\label{app:mcpp_algorithm}

\begin{algorithm}[ht]
\caption{Monte Carlo Portfolio Planning (MCPP)}
\label{alg:mcpp}
\begin{algorithmic}[1]
\Require Workflow instance \(\mathcal I=(G,\mathcal M,\Phi,B,D)\), log-scale sampling widths \(\mathcal K\), simulation budget \(N_{\mathrm{sim}}\).
\Ensure Final execution outcome: \(\mathrm{SUCCESS}\) or \(\mathrm{FAILURE}\).
\State Initialize \(s=(S,b,h)\gets(\emptyset,B,D)\).
\State Construct portfolio \(\Pi_0\gets\{\pi_{m,k}:m\in\mathcal M,\ k\in\mathcal K\}\).
\While{\(S\neq V\)}
    \State \(R(S)\gets\{v\in V\setminus S:\mathrm{Pred}(v)\subseteq S\}\).
    \State \(\widetilde{\mathcal A}(s)\gets \textsc{Candidates}(s;\mathcal M,\mathcal K,\Phi)\cup\{\pi_{m,k}(s):m\in\mathcal M,\ k\in\mathcal K\}\).
    \State \(\widetilde{\mathcal A}_{\mathrm{feas}}(s)\gets\{a\in\widetilde{\mathcal A}(s):C(a)\le b,\ \Delta(a)\le h\}\).
    \If{\(\widetilde{\mathcal A}_{\mathrm{feas}}(s)=\emptyset\)}
        \State \Return \(\mathrm{FAILURE}\).
    \EndIf
    \For{\textbf{each} \(a\in\widetilde{\mathcal A}_{\mathrm{feas}}(s)\)}
        \State \(\widehat Q_{\Pi_0}(s,a)\gets \max_{\mu\in\Pi_0}\textsc{MCValue}(s,a,\mu,N_{\mathrm{sim}};\Phi)\).
    \EndFor
    \State \(a^\star\gets\arg\max_{a\in\widetilde{\mathcal A}_{\mathrm{feas}}(s)}\widehat Q_{\Pi_0}(s,a)\).
    \State Execute \(a^\star\) in the real workflow and observe completed subtasks \(U\subseteq R(S)\).
    \State Update \(S\gets S\cup U,\quad b\gets b-C(a^\star),\quad h\gets h-\Delta(a^\star)\).
\EndWhile
\State \Return \(\mathrm{SUCCESS}\).

\Function{MCValue}{$s,a,\mu,N_{\mathrm{sim}};\Phi$}
    \State \(n_{\mathrm{succ}}\gets 0\).
    \For{\(i=1\) \textbf{to} \(N_{\mathrm{sim}}\)}
        \State \(\tilde{s}\gets\textsc{SimulateOneStep}(s,a;\Phi)\).
        \State \(n_{\mathrm{succ}}\gets n_{\mathrm{succ}}+\textsc{Rollout}(\tilde{s},\mu;\Phi)\).
    \EndFor
    \State \Return \(n_{\mathrm{succ}}/N_{\mathrm{sim}}\).
\EndFunction
\end{algorithmic}
\end{algorithm}

As shown in Algorithm~\ref{alg:mcpp}, MCPP implements a closed-loop planning procedure: at each state, it constructs candidate actions, estimates their downstream constrained completion probabilities, executes only the selected action, and then replans after observing the actual outcome. 
The subroutine \textsc{MCValue} estimates the value of a candidate action by first simulating that action and then rolling out a continuation policy \(\mu\) from the portfolio. 
Here, \textsc{Rollout} returns a binary indicator of whether the simulated continuation completes the workflow within the remaining budget and deadline; if a simulated action violates the remaining constraints, the simulation is counted as a failure. 
The continuation policy \(\mu\) is therefore used only for scoring the current candidate action, while the real workflow executes only \(a^\star\), observes the completed subtasks, updates the state, and replans.

\section{Theory: Portfolio Rollout Planning Provides Safe Improvement}
\label{app:theory}

We analyze the Monte Carlo portfolio planner used in our framework.
The planner can be viewed as a one-step rollout backup: at the current workflow state, it evaluates candidate allocation actions by simulating their future consequences under a portfolio of simple continuation policies.
The goal is not to prove global optimality over the full Bellman action space, which is intractable for realistic workflows.
Instead, we show that the portfolio planner provides a safe improvement over the best base policy in the portfolio, up to Monte Carlo estimation error, candidate-set approximation error, and errors in the estimated success rates and generation lengths.

\subsection{Setup: states, base policies, and rollout values}

Consider a workflow execution state
\[
    s=(S,b,h),
\]
where \(S\) is the set of completed subtasks, \(b\) is the remaining budget, and \(h\) is the remaining time before the deadline.
Let \(R(S)\) denote the set of currently executable subtasks.

Each allocation action assigns a base model and a parallel sampling width to every executable subtask:
\[
    a=\{(m_v,k_v)\}_{v\in R(S)},
    \qquad
    m_v\in \mathcal M,\quad
    k_v\in\mathcal K,
\]
where \(\mathcal M\) is the candidate model set and \(\mathcal K\) is a finite log-scale set of positive sampling widths.
The log-scale grid reflects the diminishing marginal return of parallel samples.
If a single attempt succeeds with probability \(p\), then \(k\) independent samples succeed with probability
\[
    q(k)=1-(1-p)^k,
\]
whose marginal gain satisfies
\[
    q(k+1)-q(k)=p(1-p)^k.
\]
Thus, larger sampling widths provide progressively smaller additional gains, making exponentially spaced widths a compact coverage of conservative-to-aggressive allocation regimes.
The concrete choice of \(\mathcal K\) is specified in the experimental setup.

Let \(\widetilde{\mathcal A}(s)\) be the candidate action set evaluated by the planner at state \(s\).
The planner also uses a portfolio of base continuation policies:
\[
    \Pi_0
    =
    \{\pi_{m,k}:m\in\mathcal M,\;k\in\mathcal K\}.
\]
Each \(\pi_{m,k}\) is a model-specific Seq-\(k\) policy.
At every future state \(s'=(S',b',h')\), it assigns the same model \(m\) and sampling width \(k\) to every executable subtask:
\[
    \pi_{m,k}(s')
    =
    \{(m,k)\}_{v\in R(S')}.
\]
If this action violates the remaining budget or deadline, its value is defined to be zero.

For any policy \(\mu\in\Pi_0\), let \(V^\mu(s)\) denote its constrained completion probability from state \(s\):
\[
    V^\mu(s)
    =
    \Pr_\mu
    \big(
        \text{the workflow completes within remaining budget } b
        \text{ and remaining time } h
        \mid s
    \big).
\]
For a candidate current action \(a\in\widetilde{\mathcal A}(s)\), define the rollout value under continuation policy \(\mu\) as
\[
    Q_\mu(s,a)
    =
    \mathbb E_{s'\sim P(\cdot\mid s,a)}
    \left[
        V^\mu(s')
    \right],
\]
where \(P(\cdot\mid s,a)\) is the stochastic transition induced by executing \(a\).
The portfolio value of action \(a\) is
\[
    Q_{\Pi_0}(s,a)
    =
    \max_{\mu\in\Pi_0} Q_\mu(s,a).
\]

\subsection{Exact portfolio improvement}

The key requirement is that the candidate action set contains the first-step actions of all base policies.

\begin{assumption}[Base actions are included]
\label{assump:base_actions}
For every \(\mu\in\Pi_0\), the action \(\mu(s)\) is included in \(\widetilde{\mathcal A}(s)\).
If \(\mu(s)\) violates the remaining budget or deadline, it is still treated as a candidate action with value zero.
\end{assumption}

\begin{definition}[Exact portfolio planner]
The exact portfolio planner chooses
\[
    a_{\Pi_0}(s)
    \in
    \arg\max_{a\in\widetilde{\mathcal A}(s)}
    Q_{\Pi_0}(s,a)
    =
    \arg\max_{a\in\widetilde{\mathcal A}(s)}
    \max_{\mu\in\Pi_0} Q_\mu(s,a).
\]
After executing \(a_{\Pi_0}(s)\), the planner observes the actual stochastic outcome, updates the workflow state, and replans.
\end{definition}

\begin{theorem}[State-wise safe improvement over the portfolio]
\label{thm:statewise_safe_improvement}
Under Assumption~\ref{assump:base_actions}, the exact portfolio planner satisfies
\[
    Q_{\Pi_0}\bigl(s,a_{\Pi_0}(s)\bigr)
    \ge
    \max_{\mu\in\Pi_0}
    V^\mu(s).
\]
\end{theorem}

\begin{proof}
By definition,
\[
    Q_{\Pi_0}\bigl(s,a_{\Pi_0}(s)\bigr)
    =
    \max_{a\in\widetilde{\mathcal A}(s)}
    \max_{\nu\in\Pi_0}
    Q_\nu(s,a).
\]
For any \(\mu\in\Pi_0\), Assumption~\ref{assump:base_actions} implies that \(\mu(s)\in\widetilde{\mathcal A}(s)\).
By the Bellman identity for the fixed policy \(\mu\),
\[
    V^\mu(s)
    =
    Q_\mu(s,\mu(s)).
\]
Therefore,
\[
    Q_{\Pi_0}\bigl(s,a_{\Pi_0}(s)\bigr)
    \ge
    Q_\mu(s,\mu(s))
    =
    V^\mu(s).
\]
Taking the maximum over \(\mu\in\Pi_0\) proves the claim.
\end{proof}

Theorem~\ref{thm:statewise_safe_improvement} shows that the exact portfolio backup cannot be worse than simply following the best base policy in \(\Pi_0\) from the current state.
The planner may choose an action that is not itself a Seq-\(k\) action, but the value of that action, evaluated under the best continuation policy in the portfolio, is at least the value of the best portfolio policy.

\begin{assumption}[Finite execution horizon]
\label{assump:finite_horizon}
The execution process has a finite maximum remaining horizon.
Equivalently, from every non-terminal state, the process terminates after at most \(H_{\max}<\infty\) additional decision rounds under any feasible policy.
\end{assumption}

We can also interpret the exact result in the closed-loop setting.
Let \(\pi_{\mathrm{exact}}\) be the policy that applies the exact portfolio planner at every visited state.

\begin{theorem}[Closed-loop portfolio improvement]
\label{thm:closed_loop_improvement}
Assume Assumptions~\ref{assump:base_actions} and~\ref{assump:finite_horizon}.
Then for any state \(s\),
\[
    V^{\pi_{\mathrm{exact}}}(s)
    \ge
    \max_{\mu\in\Pi_0} V^\mu(s).
\]
\end{theorem}

\begin{proof}
We prove the claim by backward induction on the maximum number of remaining execution rounds.
The statement is immediate at terminal states.

Consider a non-terminal state \(s\).
Let
\[
    a^\star
    =
    a_{\Pi_0}(s)
\]
be the exact portfolio action selected at \(s\), and let \(\mu^\star\in\Pi_0\) be a continuation policy attaining
\[
    Q_{\Pi_0}(s,a^\star)
    =
    Q_{\mu^\star}(s,a^\star).
\]
After executing \(a^\star\), the next state \(s'\) is drawn from \(P(\cdot\mid s,a^\star)\).
By the induction hypothesis,
\[
    V^{\pi_{\mathrm{exact}}}(s')
    \ge
    V^{\mu^\star}(s')
\]
for every next state \(s'\).
Therefore,
\[
\begin{aligned}
    V^{\pi_{\mathrm{exact}}}(s)
    &=
    \mathbb E_{s'\sim P(\cdot\mid s,a^\star)}
    \left[
        V^{\pi_{\mathrm{exact}}}(s')
    \right] \\
    &\ge
    \mathbb E_{s'\sim P(\cdot\mid s,a^\star)}
    \left[
        V^{\mu^\star}(s')
    \right] \\
    &=
    Q_{\mu^\star}(s,a^\star) \\
    &=
    Q_{\Pi_0}(s,a^\star) \\
    &\ge
    \max_{\mu\in\Pi_0} V^\mu(s),
\end{aligned}
\]
where the last inequality follows from Theorem~\ref{thm:statewise_safe_improvement}.
\end{proof}

The closed-loop result applies to the exact portfolio planner.
For the finite-sample planner, we provide a state-wise guarantee below.
A full closed-loop finite-sample guarantee would require uniform concentration over all states visited by the planner.

\subsection{Finite-sample Monte Carlo estimation}
\label{app: finite}
In practice, the rollout values \(Q_\mu(s,a)\) are not available exactly.
For each candidate pair \((a,\mu)\), the planner runs \(N\) independent simulations.
Let
\[
    Y_i(a,\mu)
    =
    \mathbf 1[
        \text{simulation } i
        \text{ completes the workflow within the remaining budget and deadline}
    ].
\]
Then
\[
    \widehat Q_\mu(s,a)
    =
    \frac{1}{N}
    \sum_{i=1}^{N} Y_i(a,\mu)
\]
is an unbiased estimator of \(Q_\mu(s,a)\).

The empirical portfolio value of action \(a\) is
\[
    \widehat Q_{\Pi_0}(s,a)
    =
    \max_{\mu\in\Pi_0}
    \widehat Q_\mu(s,a).
\]
The finite-sample planner selects
\[
    \widehat a
    \in
    \arg\max_{a\in\widetilde{\mathcal A}(s)}
    \widehat Q_{\Pi_0}(s,a).
\]
Let
\[
    \widehat\mu
    \in
    \arg\max_{\mu\in\Pi_0}
    \widehat Q_\mu(s,\widehat a)
\]
be the continuation policy attaining the empirical score for \(\widehat a\).
Only \(\widehat a\) is executed in the real workflow; \(\widehat\mu\) is used only to score the current action.

Let
\[
    L(s)
    =
    |\widetilde{\mathcal A}(s)|\,|\Pi_0|
\]
be the number of action--continuation pairs evaluated at state \(s\).

\begin{theorem}[Finite-sample state-wise guarantee]
\label{thm:finite_sample_statewise}
Fix a state \(s\) and confidence level \(\delta\in(0,1)\).
With probability at least \(1-\delta\),
\[
    Q_{\widehat\mu}(s,\widehat a)
    \ge
    \max_{a\in\widetilde{\mathcal A}(s),\,\mu\in\Pi_0}
    Q_\mu(s,a)
    -
    2\epsilon(s,\delta),
\]
where
\[
    \epsilon(s,\delta)
    =
    \sqrt{
        \frac{
            \log(2L(s)/\delta)
        }{
            2N
        }
    }.
\]
Consequently, under Assumption~\ref{assump:base_actions},
\[
    Q_{\widehat\mu}(s,\widehat a)
    \ge
    \max_{\mu\in\Pi_0}V^\mu(s)
    -
    2\epsilon(s,\delta).
\]
\end{theorem}

\begin{proof}
For any fixed pair \((a,\mu)\), the variables \(Y_i(a,\mu)\) are independent Bernoulli random variables with mean \(Q_\mu(s,a)\).
By Hoeffding's inequality,
\[
    \Pr\left(
        \left|
            \widehat Q_\mu(s,a)-Q_\mu(s,a)
        \right|
        >
        \epsilon
    \right)
    \le
    2\exp(-2N\epsilon^2).
\]
Applying a union bound over all \(L(s)\) pairs, with probability at least \(1-\delta\), all estimates satisfy
\[
    \left|
        \widehat Q_\mu(s,a)-Q_\mu(s,a)
    \right|
    \le
    \epsilon(s,\delta).
\]
Let
\[
    (a^\star,\mu^\star)
    \in
    \arg\max_{a\in\widetilde{\mathcal A}(s),\,\mu\in\Pi_0}
    Q_\mu(s,a).
\]
On the high-probability event,
\[
\begin{aligned}
    Q_{\widehat\mu}(s,\widehat a)
    &\ge
    \widehat Q_{\widehat\mu}(s,\widehat a)
    -
    \epsilon(s,\delta) \\
    &\ge
    \widehat Q_{\mu^\star}(s,a^\star)
    -
    \epsilon(s,\delta) \\
    &\ge
    Q_{\mu^\star}(s,a^\star)
    -
    2\epsilon(s,\delta).
\end{aligned}
\]
This proves the first statement.
The second statement follows from Theorem~\ref{thm:statewise_safe_improvement}.
\end{proof}

Theorem~\ref{thm:finite_sample_statewise} gives the finite-sample counterpart of the exact safe-improvement result.
The selected action is within a standard Monte Carlo estimation error of the best evaluated action--continuation pair, and therefore within the same error of the best base policy in the portfolio.

\subsection{Approximation to the exact Bellman backup}

The finite-sample theorem is stated for the candidate action set \(\widetilde{\mathcal A}(s)\) and the portfolio continuation class \(\Pi_0\).
We now make explicit how these approximations relate to the exact Bellman backup over the full log-scale action space.

Let \(\mathcal A_{\mathcal K}(s)\) denote the full log-scale action space in which each executable subtask chooses a model \(m\in\mathcal M\) and a sampling width \(k\in\mathcal K\).
The candidate action set satisfies
\[
    \widetilde{\mathcal A}(s)\subseteq \mathcal A_{\mathcal K}(s).
\]
Define the candidate-set gap
\[
    \eta(s)
    =
    \max_{a\in\mathcal A_{\mathcal K}(s),\,\mu\in\Pi_0}
    Q_\mu(s,a)
    -
    \max_{a\in\widetilde{\mathcal A}(s),\,\mu\in\Pi_0}
    Q_\mu(s,a).
\]
By definition, \(\eta(s)\ge 0\), and \(\eta(s)=0\) whenever \(\widetilde{\mathcal A}(s)=\mathcal A_{\mathcal K}(s)\).

\begin{corollary}[Candidate-set gap]
\label{cor:candidate_gap}
Under the conditions of Theorem~\ref{thm:finite_sample_statewise}, with probability at least \(1-\delta\),
\[
    Q_{\widehat\mu}(s,\widehat a)
    \ge
    \max_{a\in\mathcal A_{\mathcal K}(s),\,\mu\in\Pi_0}
    Q_\mu(s,a)
    -
    \eta(s)
    -
    2\epsilon(s,\delta).
\]
\end{corollary}

\begin{proof}
The result follows by substituting the definition of \(\eta(s)\) into Theorem~\ref{thm:finite_sample_statewise}.
\end{proof}

The gap \(\eta(s)\) measures the loss from not evaluating the full log-scale action space.
This is still a gap within the portfolio-rollout objective, because the continuation value is restricted to \(\Pi_0\).
We next isolate the additional gap between portfolio continuation and the exact Bellman continuation.

Let \(V^\star_{\mathcal K}(s)\) denote the optimal constrained completion probability over the full log-scale action space \(\mathcal A_{\mathcal K}\).
For any first action \(a\in\mathcal A_{\mathcal K}(s)\), define the exact Bellman action value
\[
    Q^\star_{\mathcal K}(s,a)
    =
    \mathbb E_{s'\sim P(\cdot\mid s,a)}
    \left[
        V^\star_{\mathcal K}(s')
    \right].
\]
The best log-scale Bellman action value is
\[
    Q^\star_{\mathcal K}(s)
    =
    \max_{a\in\mathcal A_{\mathcal K}(s)}
    Q^\star_{\mathcal K}(s,a).
\]
Because \(V^\star_{\mathcal K}(s')\ge V^\mu(s')\) for every \(\mu\in\Pi_0\), we have
\[
    Q^\star_{\mathcal K}(s,a)
    \ge
    Q_\mu(s,a)
    \qquad
    \forall a,\mu.
\]
Define the portfolio continuation gap
\[
    \zeta(s)
    =
    Q^\star_{\mathcal K}(s)
    -
    \max_{a\in\mathcal A_{\mathcal K}(s),\,\mu\in\Pi_0}
    Q_\mu(s,a).
\]
This term measures the loss from using the base-policy portfolio as the continuation value class instead of the exact optimal continuation value.
It is zero when the portfolio contains an optimal continuation policy for the best log-scale action.

\begin{corollary}[Approximation to the log-scale Bellman backup]
\label{cor:bellman_gap}
Under the conditions of Theorem~\ref{thm:finite_sample_statewise}, with probability at least \(1-\delta\),
\[
    Q^\star_{\mathcal K}(s,\widehat a)
    \ge
    Q^\star_{\mathcal K}(s)
    -
    \zeta(s)
    -
    \eta(s)
    -
    2\epsilon(s,\delta).
\]
\end{corollary}

\begin{proof}
Since \(V^\star_{\mathcal K}(s')\ge V^{\widehat\mu}(s')\) for every next state \(s'\),
\[
    Q^\star_{\mathcal K}(s,\widehat a)
    \ge
    Q_{\widehat\mu}(s,\widehat a).
\]
By Corollary~\ref{cor:candidate_gap},
\[
    Q_{\widehat\mu}(s,\widehat a)
    \ge
    \max_{a\in\mathcal A_{\mathcal K}(s),\,\mu\in\Pi_0}
    Q_\mu(s,a)
    -
    \eta(s)
    -
    2\epsilon(s,\delta).
\]
By the definition of \(\zeta(s)\),
\[
    \max_{a\in\mathcal A_{\mathcal K}(s),\,\mu\in\Pi_0}
    Q_\mu(s,a)
    =
    Q^\star_{\mathcal K}(s)
    -
    \zeta(s).
\]
Combining the inequalities proves the claim.
\end{proof}

Corollary~\ref{cor:bellman_gap} decomposes the one-step gap to the exact log-scale Bellman backup into three terms:
the portfolio continuation gap \(\zeta(s)\), the candidate-set gap \(\eta(s)\), and the Monte Carlo estimation error \(2\epsilon(s,\delta)\).
This decomposition clarifies that the planner is not claimed to be globally optimal.
Its guarantee is strongest when the base portfolio provides good continuation policies, the candidate actions cover useful allocations, and the Monte Carlo budget is sufficiently large.

\subsection{Effect of imperfect success and length estimates}

The analysis so far assumes that rollout simulations use the true success, cost, and latency statistics.
In practice, subtask success rates and generation lengths are estimated from historical executions, offline calibration, or learned predictors.
We now state a robustness bound that separates Monte Carlo estimation error from estimate-induced value error.

Let \(P\) denote the true transition kernel and let \(\widetilde P\) denote the transition kernel induced by the estimated success and length statistics used in simulation.
For any candidate pair \((a,\mu)\), let \(Q_\mu^P(s,a)\) and \(Q_\mu^{\widetilde P}(s,a)\) denote the corresponding rollout values.
Assume that the estimate-induced value discrepancy is uniformly bounded:
\[
    \left|
        Q_\mu^{\widetilde P}(s,a)
        -
        Q_\mu^P(s,a)
    \right|
    \le
    \beta(s),
    \qquad
    \forall a\in\widetilde{\mathcal A}(s),\ \mu\in\Pi_0.
\]
The quantity \(\beta(s)\) captures the value error caused by imperfect estimates of success rates, generation lengths, costs, or latencies.

\begin{theorem}[Finite-sample guarantee under estimation error]
\label{thm:estimate_error}
Suppose Monte Carlo simulations are generated using the estimated transition kernel \(\widetilde P\), and the estimate-induced value discrepancy is bounded by \(\beta(s)\).
With probability at least \(1-\delta\),
\[
    Q_{\widehat\mu}^{P}(s,\widehat a)
    \ge
    \max_{a\in\widetilde{\mathcal A}(s),\,\mu\in\Pi_0}
    Q_\mu^{P}(s,a)
    -
    2\epsilon(s,\delta)
    -
    2\beta(s).
\]
Consequently, under Assumption~\ref{assump:base_actions},
\[
    Q_{\widehat\mu}^{P}(s,\widehat a)
    \ge
    \max_{\mu\in\Pi_0}V_P^\mu(s)
    -
    2\epsilon(s,\delta)
    -
    2\beta(s).
\]
\end{theorem}

\begin{proof}
The Monte Carlo estimates concentrate around \(Q_\mu^{\widetilde P}(s,a)\).
By Theorem~\ref{thm:finite_sample_statewise}, applied under the simulated transition kernel \(\widetilde P\), with probability at least \(1-\delta\),
\[
    Q_{\widehat\mu}^{\widetilde P}(s,\widehat a)
    \ge
    \max_{a,\mu}Q_\mu^{\widetilde P}(s,a)
    -
    2\epsilon(s,\delta).
\]
Using the estimate-error bound,
\[
\begin{aligned}
    Q_{\widehat\mu}^{P}(s,\widehat a)
    &\ge
    Q_{\widehat\mu}^{\widetilde P}(s,\widehat a)
    -
    \beta(s) \\
    &\ge
    \max_{a,\mu}Q_\mu^{\widetilde P}(s,a)
    -
    2\epsilon(s,\delta)
    -
    \beta(s) \\
    &\ge
    \max_{a,\mu}Q_\mu^{P}(s,a)
    -
    2\epsilon(s,\delta)
    -
    2\beta(s).
\end{aligned}
\]
The portfolio comparison follows from Theorem~\ref{thm:statewise_safe_improvement}.
\end{proof}

Theorem~\ref{thm:estimate_error} explains the role of the noise-injection experiments.
When success rates or generation lengths are estimated imperfectly, the planner optimizes the simulated workflow model induced by these estimates.
The degradation is controlled by the induced value discrepancy \(\beta(s)\).
Empirically, we perturb predicted subtask success probabilities and generation lengths to test how sensitive the planner is to such estimation errors.

\paragraph{A sufficient condition for success-probability errors.}
A concrete sufficient condition for small \(\beta(s)\) can be stated for errors in subtask completion probabilities, assuming the cost and latency estimates are fixed.
Errors in generation length, cost, or latency can also affect feasibility boundaries.
These effects are captured by the abstract value discrepancy \(\beta(s)\) above, or can be analyzed under additional discretization or regularity assumptions on the value function.

Suppose only the subtask completion probabilities are perturbed.
If, over the remaining execution horizon \(H(s)\),
\[
    \sup_{s,a}
    \mathrm{TV}
    \left(
        P(\cdot\mid s,a),
        \widetilde P(\cdot\mid s,a)
    \right)
    \le
    \alpha,
\]
where \(\mathrm{TV}\) denotes total variation distance, then a standard simulation argument gives
\[
    \left|
        V_P^\pi(s)-V_{\widetilde P}^\pi(s)
    \right|
    \le
    \min\{1,H(s)\alpha\}
\]
for any fixed policy \(\pi\).
Thus, in Theorem~\ref{thm:estimate_error}, one may take
\[
    \beta(s)\le \min\{1,H(s)\alpha\}.
\]

For the Bernoulli completion model, suppose each attempted subtask success probability is perturbed by at most \(\rho\) after accounting for the selected model and sampling width:
\[
    \left|
        q_{v,m}(k)-\widetilde q_{v,m}(k)
    \right|
    \le
    \rho.
\]
For an action attempting \(r\) executable subtasks, a coupling argument yields
\[
    \mathrm{TV}
    \left(
        P(\cdot\mid s,a),
        \widetilde P(\cdot\mid s,a)
    \right)
    \le
    \min\{1,r\rho\}.
\]
This bound is conservative, but it shows explicitly how errors in predicted subtask success probabilities affect rollout evaluation.
Errors in generation-length estimates are handled by the more general value-discrepancy term \(\beta(s)\).

\subsection{Estimating the closed-loop completion probability}
\label{app:closed_loop_prediction}

The same Monte Carlo mechanism can be used to estimate the success probability of the full closed-loop planner.
Given a workflow, a budget \(B\), and a deadline \(D\), run \(N_{\mathrm{eval}}\) simulated executions of the complete online policy from
\[
    s_0=(\emptyset,B,D).
\]
Let
\[
    Z_i
    =
    \mathbf 1[
        \text{closed-loop simulation } i
        \text{ completes within } B,D
    ].
\]
The predicted constrained completion probability is
\[
    \widehat P_{\mathrm{succ}}
    =
    \frac{1}{N_{\mathrm{eval}}}
    \sum_{i=1}^{N_{\mathrm{eval}}}
    Z_i.
\]
Since \(Z_i\) are Bernoulli variables, Hoeffding's inequality gives
\[
    \Pr\left(
        \left|
            \widehat P_{\mathrm{succ}}
            -
            P_{\mathrm{succ}}
        \right|
        >
        \epsilon
    \right)
    \le
    2\exp(-2N_{\mathrm{eval}}\epsilon^2).
\]
Equivalently, with probability at least \(1-\delta\),
\[
    \left|
        \widehat P_{\mathrm{succ}}
        -
        P_{\mathrm{succ}}
    \right|
    \le
    \sqrt{
        \frac{\log(2/\delta)}
        {2N_{\mathrm{eval}}}
    }.
\]
This prediction mode estimates the probability that the full online planner completes the workflow within the specified budget and deadline.
When simulations use estimated success rates and generation lengths, this is an estimate under the induced simulated execution model.
Calibration against real executions is evaluated empirically.
It is distinct from the action-selection guarantees above, which compare a single planning decision or the exact closed-loop planner against the base-policy portfolio.

\section{Implementation Details}
\label{app:implementation_details}
\subsection{Offline Workflow Pool Construction}
\label{sec:pool build}
We construct empirical execution profiles through large-scale offline rollouts.
For each model--subtask pair, the rollout pool stores $512$ samples with success outcomes, wall-clock latency, and output-token length.
For each benchmark, we build the evaluation pool from the union of usable workflows obtained from the evaluated model rollouts, resulting in $103$ usable workflows for ProofFlow and $462$ usable workflows for CodeFlow.
All experiments for collecting offline model rollouts are executed in parallel on a cluster with 512 NVIDIA H800 GPUs, where each individual run is allocated to 8 GPUs.

For ProofFlow, offline rollout collection requires both formalization and proof generation.
On each $8$-H800 machine, we split the GPUs into two groups of four GPUs: one group loads the formalizer model and the other group loads the prover model.
A workflow rollout first invokes the formalizer to translate the input problem into a Lean statement, and then invokes the prover to generate the corresponding proof.
Lean~4.15.0 is utilized as the verifier to check whether the generated proof is accepted.

 To ensure reproducibility and fair comparison, we adhered to the sampling hyperparameters for each model family:
\begin{itemize}
    \item Qwen3-Thinking-Series: Temperature $T=0.6$, top-$p=0.95$, top-$k=20$, with a maximum generation length of 38k tokens.
    \item Qwen3-Instruct-Series: Temperature $T=0.7$, top-$p=0.8$, top-$k=20$, with a maximum generation length of 38k tokens.
    \item Goedel-V2-Series: Temperature $T=0.6$, top-$p=0.95$, top-$k=20$, with a maximum generation length of 16k tokens.
    \item Kimina-7B-Series: Temperature $T=0.6$, top-$p=0.95$, top-$k=0$, with a maximum generation length of 8k tokens.
\end{itemize}

\subsection{Online Simulation}
\label{sec:Onlinesim}
All simulator evaluations of the proposed MCPP method are CPU-only and are conducted on machines equipped with AMD EPYC 9T24 96-Core processors with 96 CPU cores.
The inner Monte Carlo rollout evaluator is implemented with a Numba-accelerated kernel to reduce Python overhead. 
For ProofFlow, we choose three model pairs as our candidate model set:
\textsc{Goedel-Formalizer-V2-32B} with \textsc{Goedel-Prover-V2-32B},
\textsc{Goedel-Formalizer-V2-8B} with \textsc{Goedel-Prover-V2-8B} \citep{lin2025goedelproverv2},
and \textsc{Kimina-Autoformalizer-7B} with \textsc{Kimina-Prover-Preview-Distill-7B} \citep{kimina_prover_2025}.
For CodeFlow, we choose four models as our candidate model set:
\textsc{Qwen3-30B-A3B-Instruct-2507},
\textsc{Qwen3-30B-A3B-Thinking-2507},
\textsc{Qwen3-4B-Instruct-2507}, and
\textsc{Qwen3-4B-Thinking-2507} \citep{qwen3technicalreport}.
Moreover, we employ $\mathcal{K} = \{1,4,16,64\} $ as our candidate rollout-width portfolio.
We calculate budget consumption according to official output-token prices\footnote{\url{https://www.alibabacloud.com/help/en/model-studio/model-pricing}}.
% The simulator uses a fixed internal API price table, shown in Table~\ref{tab:api_prices}, to convert output tokens into monetary cost.
% All methods are evaluated with the same price table, so budget comparisons are controlled across policies.

\subsection{Noise Perturbation Details}
\label{app:robustness_details}

In the robustness experiments, we perturb only planner-visible empirical profiles while keeping the true simulator execution distribution unchanged.
This setting evaluates robustness to profile-estimation mismatch: the planner makes decisions using noisy rollout pools, whereas execution outcomes are still sampled from the original empirical pools.

For token-length noise, let \(c_i\) denote the token length of the \(i\)-th sample in the empirical pool of node \(v\) and model \(m\).
Let
\[
    c_{\max}^{(v,m)} = \max_i c_i
\]
be the maximum value in that node--model pool.
We normalize each sample by \(c_{\max}^{(v,m)}\), add Gaussian noise in the normalized space, and rescale it back:
\[
    \tilde{c}_{i}
    =
    c_{\max}^{(v,m)}
    \cdot
    \mathrm{clip}
    \left(
        \frac{c_i}{c_{\max}^{(v,m)}} + \sigma z_i,\,
        \epsilon,\,
        1-\epsilon
    \right),
    \qquad
    z_i\sim\mathcal{N}(0,1).
\]
Here, \(\tilde{c}_i\) is the perturbed planner-visible token-length sample, \(\sigma\) controls the noise level, \(z_i\) is an independent standard Gaussian perturbation for sample \(i\), and \(\epsilon\) is a small numerical constant used to avoid degenerate zero values.

For success-probability noise, let \(p_{v,m}\) be the empirical success rate of node \(v\) under model \(m\).
We perturb this rate directly in probability space:
\[
    \tilde{p}_{v,m}
    =
    \mathrm{clip}
    \left(
        p_{v,m} + \sigma z,\,
        \epsilon,\,
        1-\epsilon
    \right),
    \qquad
    z\sim\mathcal{N}(0,1),
\]
where \(\tilde{p}_{v,m}\) is the perturbed planner-visible success probability and \(z\) is a standard Gaussian perturbation sampled for the corresponding node--model pool.
Given a pool of \(n_{v,m}\) empirical samples, we then minimally flip success labels so that the number of successful samples matches approximately
\[
    \mathrm{round}\!\left(\tilde{p}_{v,m} n_{v,m}\right).
\]
Thus, the planner observes a noisy success distribution, while the simulator still executes using the original empirical outcomes.
Under this formulation, \(\sigma=0.3\) is already a strong success-probability perturbation.
For instance, for a node--model pair with \(p_{v,m}=0.5\), one standard deviation corresponds to an approximate range of \(0.2\) to \(0.8\) before clipping.

\section{Additional Analysis}
\subsection{Scaling with Workflow Size}
\label{sec:node_range}
To assess whether our method remains effective on larger workflows, we group ProofFlow workflows by the number of subtasks and report the success rates for different sizes in Figure~\ref{fig:node_range}.
Figure~\ref{fig:node_range} shows that the performance gaps between MCPP and baselines increase with the number of subtasks, especially under tight constraints. 
Since larger workflows amplify the impact of early allocation errors, static Uniform cannot condition its initial allocation on realized intermediate outcomes, and on-policy Retry fails to adapt to heterogeneous subtask difficulty, making both baselines less effective at completing larger workflows under tight deadlines.
In contrast, by updating its state after each execution event, MCPP remains responsive to the realized trajectory, maintaining a clear advantage across all workflow sizes and yielding its largest gains on the most constrained instances.

\begin{figure}[t]
    \centering
    \includegraphics[width=0.9\textwidth]{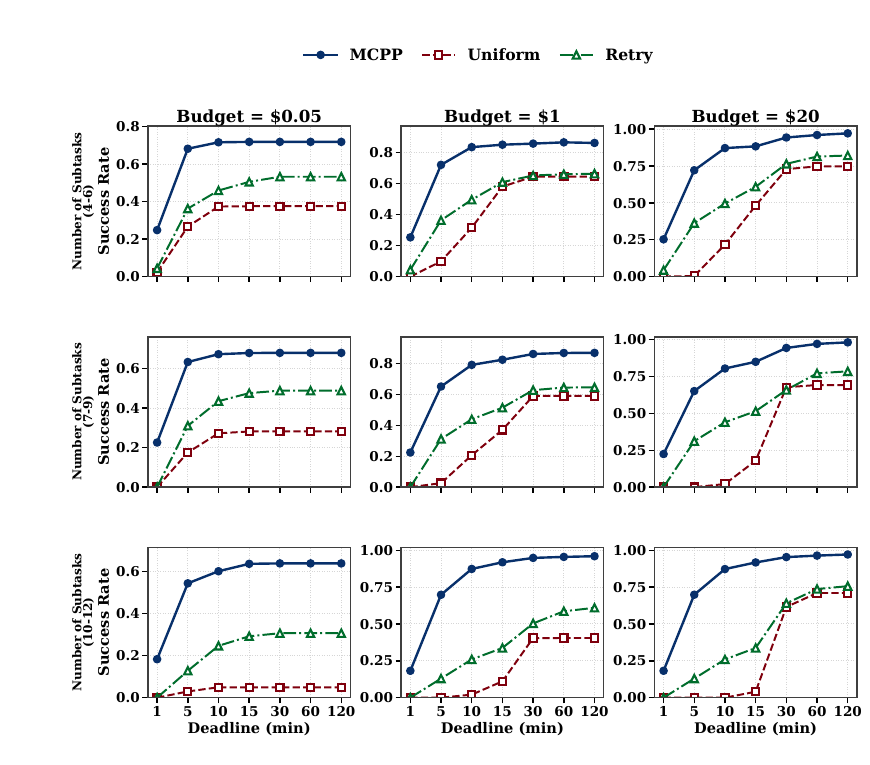}
    \caption{
    \textbf{Scaling with the size of workflows.}
    Monte Carlo Portfolio Rollout Search consistently outperforms baselines across workflow sizes, with particularly clear gains on larger and tighter instances.
    }
    \label{fig:node_range}
\end{figure}

\subsection{Runtime Analysis}
\label{sec:runtime_analysis}

To quantify the overhead of online planning, we report the mean planner wall-clock time under the default $M=64$ configuration in Table~\ref{tab:runtime_trend}.
Planning overhead remains moderate across all budget--deadline settings.
Under tight constraints, the mean planner time is around $0.6$ seconds, since only a small set of actions is feasible and simulated continuations terminate quickly.
As the budget and deadline are relaxed, more actions become feasible and Monte Carlo continuations can explore deeper into the remaining workflow, leading to a predictable increase in runtime.
However, even in the loosest setting, the mean planner time remains only a few seconds, indicating that the default implementation is practical for online receding-horizon execution.

\begin{table*}[t]
    \centering
    \footnotesize
    \renewcommand{\arraystretch}{0.92}
    \setlength{\tabcolsep}{4.8pt}
    \caption{
    Planner wall-clock time across different budget–deadline settings
    }
    \label{tab:runtime_trend}
    \begin{tabular}{lccccccc}
        \toprule
        \multirow{2}{*}{\textbf{Budget ($B$)}} &
        \multicolumn{7}{c}{\textbf{ Deadline ($D$)}} \\
        \cmidrule(lr){2-8}
        & $1$ min & $5$ min & $10$ min & $15$ min & $30$ min & $60$ min & $120$ min \\
        \midrule
        $\$0.05$ & 0.64 & 0.82 & 0.91 & 0.97 & 1.06 & 1.07 & 1.11 \\
        $\$1$    & 0.60 & 0.88 & 1.07 & 1.30 & 1.79 & 2.38 & 3.09 \\
        $\$20$   & 0.60 & 0.89 & 1.09 & 1.36 & 1.96 & 2.85 & 4.30 \\
        \bottomrule
    \end{tabular}
\end{table*}

\subsection{Analysis of Monte Carlo Budgets}
\label{sec:m_sweep}
Table~\ref{tab:m_sweep_values} demonstrates constrained completion probability and mean planner wall-clock time under different Monte Carlo budgets.
Performance improves from small \(M\) to moderate \(M\), but quickly saturates, while planning time continues to increase substantially.
For example, under \(B=\$20\) and \(D=120\) minutes, increasing \(M\) from \(64\) to \(256\) improves success rate only from \(96.4\%\) to \(96.9\%\), but increases mean planner time from \(17.58\) seconds to \(66.99\) seconds.
Thus, \(M=64\) provides a strong quality--runtime trade-off across budget--deadline settings.

This saturation suggests that MCPP benefits primarily from online lookahead over downstream constrained completion, rather than sampling alone.
Once enough samples are available to distinguish high-value actions, additional simulations mostly reduce estimator variance and provide limited gains.
Therefore, although larger \(M\) can theoretically reduce the approximation gap as discussed in Appendix~\ref{app: finite}, moderate Monte Carlo budgets are sufficient in practice.
We use \(M=64\) as the default setting in the main experiments.

\begin{table*}[t]
    \small
    \centering
    \renewcommand{\arraystretch}{0.92}
    \setlength{\tabcolsep}{2.2pt}
    \caption{
    Success rate and planner wall-clock time under different Monte Carlo budgets.
    Each cell reports constrained completion probability in percentage, with mean planner wall-clock time in seconds shown in parentheses.
    }
    \label{tab:m_sweep_values}
    \resizebox{\textwidth}{!}{
    \begin{tabular}{llccccccc}
        \toprule
        \textbf{Budget} & \textbf{MC samples}
        & $D=1$ min & $D=5$ min & $D=10$ min & $D=15$ min
        & $D=30$ min & $D=60$ min & $D=120$ min \\
        \midrule
        \multirow{5}{*}{$B=\$0.05$}
        & $M=16$  & 19.9 (0.92) & 55.5 (1.03) & 61.9 (1.11) & 63.2 (1.16) & 63.4 (1.09) & 63.4 (1.16) & 63.4 (1.07) \\
        & $M=32$  & 20.2 (1.44) & 57.4 (1.75) & 62.4 (1.69) & 64.0 (1.60) & 64.1 (1.75) & 64.1 (1.77) & 64.1 (1.77) \\
        & $M=64$  & 20.5 (2.40) & 59.0 (2.84) & 63.3 (2.94) & 64.5 (2.88) & 64.6 (3.01) & 64.6 (2.82) & 64.6 (3.02) \\
        & $M=128$ & 20.3 (4.30) & 59.8 (5.14) & 64.0 (5.28) & 65.3 (5.50) & 65.4 (5.37) & 65.4 (5.35) & 65.4 (5.24) \\
        & $M=256$ & 20.4 (8.18) & 60.3 (9.72) & 64.4 (10.11) & 65.6 (10.16) & 65.8 (10.28) & 65.8 (10.13) & 65.8 (10.18) \\
        \midrule
        \multirow{5}{*}{$B=\$1$}
        & $M=16$  & 20.0 (1.24) & 60.8 (2.33) & 75.6 (2.76) & 79.9 (2.92) & 84.5 (3.07) & 85.6 (3.21) & 86.3 (3.21) \\
        & $M=32$  & 20.1 (2.21) & 64.0 (3.85) & 78.5 (5.02) & 81.6 (5.36) & 84.9 (5.63) & 86.1 (5.72) & 86.5 (5.81) \\
        & $M=64$  & 20.6 (3.85) & 65.8 (7.29) & 79.7 (9.35) & 82.8 (9.88) & 85.7 (10.54) & 86.6 (10.95) & 86.6 (11.10) \\
        & $M=128$ & 20.4 (7.24) & 67.2 (14.07) & 80.7 (17.90) & 83.4 (19.35) & 86.1 (20.18) & 86.7 (20.81) & 86.7 (21.49) \\
        & $M=256$ & 20.5 (14.19) & 67.6 (27.65) & 81.0 (36.63) & 83.5 (38.10) & 86.1 (39.72) & 86.7 (42.68) & 86.8 (43.11) \\
        \midrule
        \multirow{5}{*}{$B=\$20$}
        & $M=16$  & 20.0 (1.35) & 60.9 (2.35) & 77.2 (3.62) & 81.8 (3.84) & 89.9 (4.20) & 92.8 (4.54) & 94.7 (4.83) \\
        & $M=32$  & 20.1 (2.16) & 64.0 (4.45) & 79.9 (6.57) & 83.5 (7.01) & 91.4 (7.96) & 93.9 (8.46) & 95.9 (9.07) \\
        & $M=64$  & 20.6 (4.04) & 65.8 (8.05) & 81.0 (12.48) & 84.7 (13.79) & 92.5 (15.84) & 95.0 (16.40) & 96.4 (17.58) \\
        & $M=128$ & 20.4 (7.41) & 67.2 (15.99) & 82.0 (24.50) & 85.6 (27.04) & 92.8 (30.13) & 95.9 (32.35) & 96.8 (34.50) \\
        & $M=256$ & 20.5 (14.38) & 67.6 (29.99) & 82.4 (49.16) & 86.1 (51.57) & 93.2 (60.87) & 96.2 (64.39) & 96.9 (66.99) \\
        \bottomrule
    \end{tabular}
    }
\end{table*}

\section{Prompts}
In this section, we present the full prompt templates used in our experiments. 
\begin{figure*}[ht]
    \centering
    \begin{tcolorbox}[
        colback=gray!5,
        colframe=black,
        boxrule=0.8pt,
        arc=2pt,
        left=10pt, right=10pt, top=10pt, bottom=10pt,
        title=\textbf{CodeFlow Multi-Turn Input Serialization Template}
    ]
        \ttfamily
        \scriptsize

        System Message: \\
        You are a Programming Expert. You always provide correct and reliable code solutions.

        \vspace{1em}

        User Message: \\

        You are a Programming Expert. You always provide correct and reliable code solutions. You will be provided with the Background of the whole problem, a programming problem and may also some pre-implemented functions. If pre-implemented functions provided, you need to call the pre-implemented functions and write a new function to solve the problem.

        \vspace{1em}

        Background of the whole problem: \\
        \textcolor{red}{\{\{ problem\_description \}\}}

        \vspace{1em}

        Problem Description: \\
        You need to complete \textcolor{red}{\{\{ name \}\}} function. \\
        \textcolor{red}{\{\{ statement \}\}}

        \vspace{1em}

        \textcolor{gray}{[Optional Dependency Block]} \\
        Dependency information: \\
        To solve the problem, you need to utilize the pre-implemented functions
        \textcolor{red}{\{\{ dependencies \}\}} provided.

        \vspace{0.5em}

        Pre-implemented functions: \\
        \textcolor{red}{\{\{ history \}\}}

        \vspace{1em}

        Guidelines: \\
        - Ensure the function is executable and meets the requirement. \\
        - Handle dependency information correctly when dependencies are provided. \\
        - Provide clear and concise comments to explain key parts of the code.

        \vspace{1em}

        \textcolor{gray}{[Final-Turn Prefix]} \\
        For the final subproblem, generate code beginning with: \\
        import sys \\
        def \textcolor{red}{\{\{ name \}\}}(): \\
        \hspace*{2em}input = sys.stdin.read().split()

        \vspace{1em}

        Return your response by filling the function body following the function signature provided. Just generate the function itself and do not output examples.

        \vspace{0.5em}

        ```python
    \end{tcolorbox}
    \caption{\textbf{CodeFlow multi-turn input serialization template.} The current subproblem is serialized with the whole-problem background, the target function name, the subproblem statement, and optionally dependency names and previously generated implementations. The dependency and final-turn blocks are included according to the subproblem position in the workflow.}
    \label{fig:codeflow_prompt}
\end{figure*}

\begin{figure*}[ht]
    \centering
    \begin{tcolorbox}[
        colback=gray!5,
        colframe=black,
        boxrule=0.8pt,
        arc=2pt,
        left=10pt, right=10pt, top=10pt, bottom=10pt,
        title=\textbf{ProofFlow Formalizer Input Serialization Template}
    ]
        \ttfamily
        \scriptsize
        \raggedright

        System Message:\par
        You are a thinking model specialized in turning natural-language math statements into Lean 4 code.

        \par\medskip

        User Message:\par
        Please autoformalize the following natural language problem proof step in Lean 4.

        \par\medskip

        Use the following lemma name:\par
        \textcolor{red}{\{\{ lemma\_header \}\}}

        \par\medskip

        The natural language statement is:\par
        \textcolor{red}{\{\{ item.statement \}\}}

        \par\medskip

        The dependencies are:\par
        \textcolor{red}{\{\{ dependencies \}\}}

        \par\medskip

        This is the Lean code skeleton you need to use:

        \par\smallskip
        \textcolor{gray}{--- lean4 ---}\par
        import Mathlib\par
        import Aesop\par
        \par
        set\_option maxHeartbeats 0\par
        \par
        open BigOperators Real Nat Topology Rat Filter\par
        \par
        \textcolor{red}{\{\{ lemma\_header \}\}}\par
        {[}place correct hypothesis here{]} :\par
        {[}place goal here{]} := by\par
        sorry\par
        \textcolor{gray}{--- end lean4 ---}

        \par\medskip

        \textcolor{gray}{[Optional Verified Dependency Context]}\par
        The following Lean 4 code contains verified declarations and proofs you may use:\par
        \textcolor{red}{\{\{ dependency\_context\_code \}\}}

        \par\medskip

        Important: Please write only one lemma or theorem.
    \end{tcolorbox}
    \caption{\textbf{ProofFlow formalizer input serialization template.} Each proof-graph node is converted into a Lean formalization task using the node statement, dependency identifiers, and a fixed Lean skeleton. When available, verified prior Lean declarations are appended as dependency context.}
    \label{fig:proofflow_formalizer_prompt}
\end{figure*}

\begin{figure*}[ht]
    \centering
    \begin{tcolorbox}[
        colback=gray!5,
        colframe=black,
        boxrule=0.8pt,
        arc=2pt,
        left=10pt, right=10pt, top=10pt, bottom=10pt,
        title=\textbf{ProofFlow Solver Input Serialization Template}
    ]
        \ttfamily
        \scriptsize

        System Message: \\
        You are an expert Lean 4 theorem prover. Your job is to complete partially written Lean 4 code and provide correct, verifiable proofs.

        \vspace{1em}

        User Message: \\

        This is the lemma/theorem I want you to prove: \\
        \textcolor{red}{\{\{ item.statement \}\}}

        \vspace{1em}

        Complete the following Lean 4 code. Do not remove imports:

        \vspace{0.5em}

        ```lean4 \\
        \textcolor{red}{\{\{ item.formalization["lean\_code"] \}\}} \\
        ```

        \vspace{1em}

        You can adapt previous Lean 4 lemma statements to fit the goal, especially if you encounter errors.

        \vspace{1em}

        \textcolor{gray}{[Optional Verified Dependency Context]} \\
        The following Lean 4 code contains already verified previous proof steps and declarations that this node may depend on: \\
        \textcolor{red}{\{\{ dependency\_context\_code \}\}}

        \vspace{1em}

        \textcolor{gray}{[Optional Error Notice]} \\
        The previous Lean 4 code contains errors. Please take that into account.

    \end{tcolorbox}
    \caption{\textbf{ProofFlow solver input serialization template.} After formalization, the solver receives the target proof-step statement, the generated Lean skeleton, and optionally verified dependency code. The model is instructed to replace the placeholder proof with a complete Lean 4 proof without \texttt{sorry}.}
    \label{fig:proofflow_solver_prompt}
\end{figure*}

% \clearpage
% \input{checklist.tex}

\end{document}